\documentclass{article}
\usepackage[utf8]{inputenc}

\title{Making Progress Based on False Discoveries  }
\author{Roi Livni\\ \small{Tel Aviv University}\\\small{\textit{rlivni@tauex.tau.ac.il}}}

\usepackage[utf8]{inputenc}
\usepackage[margin=1in]{geometry}

\usepackage{amsmath,amsfonts,amsthm,amssymb}
\usepackage{mathtools}

\usepackage[numbers]{natbib} 
\usepackage{thmtools,thm-restate}
\usepackage[hidelinks,colorlinks=true,linkcolor=blue,citecolor=blue]{hyperref}
\usepackage[capitalise]{cleveref}

\declaretheorem[parent=section]{theorem}
\declaretheorem[parent=section]{lemma}
\declaretheorem[sibling=theorem]{corollary}

\declaretheorem[parent=section]{definition}

\newtheorem*{theorem*}{Theorem}

\declaretheoremstyle[
        spaceabove=\topsep, 
        spacebelow=\topsep, 
        bodyfont=\normalfont,
        notefont=\normalfont\bfseries,
        notebraces={}{},
        qed=$\blacksquare$, 
    ]{proofstyle}
\declaretheorem[style=proofstyle,numbered=no,name=Proof]{proof}

\usepackage{xcolor}
\usepackage[extdef=true]{delimset}
\usepackage{xspace}
\usepackage{crossreftools}

\crefname{claim}{Claim}{Claims}

\newcommand{\reals}{\mathbb{R}}

\newcommand{\E}{\mathop{\mathbb{E}}}

\DeclareMathOperator*{\argmax}{\arg\max}

\DeclarePairedDelimiter\floor{\lfloor}{\rfloor}

\renewcommand{\epsilon}{\varepsilon}

\newcommand{\ignore}[1]{}

\newcommand{\Q}{\mathcal{Q}}

\newcommand{\X}{\mathcal{X}}
\newcommand{\B}{\mathcal{B}}
\renewcommand{\O}{\mathcal{O}}

\newcommand{\q}{\mathbf{q}}
\renewcommand{\a}{\mathbf{a}}
\renewcommand{\b}{\mathbf{b}}
\renewcommand{\S}{\mathcal{S}}
\newcommand{\D}{\mathcal{D}}
\newcommand{\A}{\mathcal{A}}

\newcommand{\T}{\mathcal{T}}

\newcommand{\fla}{FOA\xspace}

\begin{document}
\maketitle
\begin{abstract}
The study of adaptive data analysis examines how many statistical queries can be answered accurately using a fixed dataset while avoiding false discoveries (statistically inaccurate answers). In this paper, we tackle a question that precedes the field of study: Is data only valuable when it provides accurate answers to statistical queries? To answer this question, we use Stochastic Convex Optimization as a case study.

In this model, algorithms are considered as analysts who query an estimate of the gradient of a noisy function at each iteration and move towards its minimizer. It is known that $O(1/\epsilon^2)$ examples can be used to minimize the objective function, but none of the existing methods depend on the accuracy of the estimated gradients along the trajectory.
Therefore, we ask: How many samples are needed to minimize a noisy convex function if we require $\epsilon$-accurate estimates of $O(1/\epsilon^2)$ gradients? Or, might it be that inaccurate gradient estimates are \emph{necessary} for finding the minimum of a stochastic convex function at an optimal statistical rate?

We provide two partial answers to this question. First, we show that a general analyst (queries that may be maliciously chosen) requires $\Omega(1/\epsilon^3)$ samples, ruling out the possibility of a foolproof mechanism. Second, we show that, under certain assumptions on the oracle, $\tilde \Omega(1/\epsilon^{2.5})$ samples are necessary for gradient descent to interact with the oracle. Our results are in contrast to classical bounds that show that $O(1/\epsilon^2)$ samples can optimize the population risk to an accuracy of $O(\epsilon)$, but with spurious gradients.
\end{abstract}

\section{Introduction}
Adaptive data analysis is a recent mathematical framework \cite{dwork2015preserving} designed to address issues with false discoveries in modern data analysis \cite{ioannidis2005contradicted, ioannidis2005most, prinz2011believe}. The framework aims to prevent overfitting that can result from the reuse of data and adaptivity of the analysis by modeling the interaction between an analyst and a fixed dataset through an oracle or mechanism. 
In this setup, the analyst is assumed to act maliciously, with the aim of uncovering such false discoveries, i.e. queries for which the oracle gives incorrect answers. Meanwhile, the oracle is designed to provide valid answers. This formalism has led to the development of new algorithms and techniques that allow for more responsible use of data \cite{bassily2020stability, dwork2015reusable, blum2015ladder, dwork2015Generalization}. Additionally, research in this area has also revealed the limitations of what can be achieved with limited resources \cite{ullman2018limits, hardt2014preventing, steinke2015interactive}.

An appealing example, often studied \cite{feldman2017statistical, zrnic2019natural}, that demonstrates the concept is Stochastic Convex Optimization (SCO) \cite{shalev2009stochastic}. The goal of SCO is to find the minimum, within a bounded set, of a convex function through access to noisy (Lipschitz and convex) samples of it. A standard approach is an iterative process, where the algorithm starts at a point $w_1$ and updates $w_t$ with gradient estimates. For instance, full-batch Gradient Descent uses the following formula for updates: \begin{equation}\label{eq:gdintro} w_{t+1} = \Pi\left(w_t -\frac{\eta}{m}\sum_{i=1}^m \nabla f_i (w_t)\right),\end{equation} where $f_i$ are i.i.d noisy samples of the function, $\Pi$ is a projection operator, and the estimates are given by the empirical mean. This approach fits the framework of adaptive data analysis, as the optimizer (analyst) queries the gradient of the population risk at $w_t$, receives an estimation from the dataset, and updates its state accordingly. Notice that, distinctively from the general setup, the analyst isn't necessarily malicious in any way.

An important question arises when thinking of Gradient Descent as a tool to minimize the population loss: Is it essential to avoid adaptivity in the gradient-estimates?
%
%
Perhaps the first result that comes to mind in this context is that of Stochastic Gradient Descent (SGD): In SGD the algorithm uses highly noisy estimates of the gradient. Instead of taking the empirical mean, we just sample one point (without replacement) as an estimate. On the one hand, it avoids the problem of adaptivity by using few examples per iteration. On the other hand, it does not even try to provide correct gradients. From the optimization point of view this algorithm achieves the optimal statistical rates  \cite{nemirovskij1983problem}.

But there are algorithms that don't necessarily avoid adaptiveness and reuse the data to estimate the gradient. For example, full-batch GD as depicted in \cref{eq:gdintro}. How do they perform, and how does the problem of adaptivity affect them?
Perhaps as expected, adaptivity does come with a certain cost. A recent construction by \citet{amir2021sgd} shows that GD, with standard choice of hyperparameters (i.e. learning rate and no. of iterations) can minimize the empirical risk, and at the same time overfit and fail to generalize. A close examination of the construction shows that, already in the second step, the gradient starts to be biased and does not represent the true gradient of the loss function. In a subsequent work, \citep{amir2021never}, it was shown that any method that only has access to the empirical gradient cannot achieve optimal (computational) rates. This is perhaps a good example to the shortcoming of naive reuse of data.

However, from the optimization point of view, there is a remedy available (albeit, sub-optimal computational rates). \citet{bassily2020stability} demonstrated that by using a smaller step size and increasing the number of iterations, the algorithm can be stabilized while still maintaining an optimal sample complexity rate. However, from the perspective of adaptive data analysis, this solution may be deemed unacceptable. In fact, one can show that the solution does not make the gradients any more accurate (see \cref{thm:amir}). Instead, the solution stabilizes the analyst and it increases the number of queries made by the analyst, meaning that it interacts with with an inaccurate mechanism even more intensely. Adaptive data analysis is, in a sense, "unaware" of this solution.

But is there another fix within the framework of adaptive data analysis? Let us consider Gradient Descent as any algorithm that makes adaptive steps as in \cref{eq:gdintro} but with any estimate of the gradient, not necessarily empirical mean. Then the question we would like to answer is:
\begin{center}
What is the sample complexity of providing $O(1/\epsilon^2)$, $\epsilon$-accurate gradients of a $1$-Lipschitz convex function to Gradient Descent with learning rate $\eta =O(\epsilon)$? 
\end{center}

If the solution requires more than $O(1/\epsilon^2)$ samples, it means that, counterintuitively, the analyst must observe incorrect estimates of the gradient in order to accurately locate the minimum of the expected function.
We require $O(1/\epsilon^2)$ gradients and $O(\epsilon)$ learning rate as these are known to be necessary for optimization \cite{nesterov2003introductory, blair1985problem} (and it is also easy to see that it is sufficient). We focus here on dimension-independent bounds as these are the optimal achievable rates. It is easy to see that $\tilde O(1/\epsilon^4)$ is a naive, dimension independent bound that one could achieve (where the oracle uses $O(1/\epsilon^2$)-fresh new samples at each iteration, hence by standard dimension independent concentration bounds \cite{boucheron2013concentration}). Standard techniques of adaptive data analysis can also be used to achieve rates of $\tilde{O}(\sqrt{d}/\epsilon^3)$ \cite{bassily2021algorithmic} but this is both dimension dependent and remains suboptimal for optimization purposes.
The question above remains open, but we provide two intermediate answers, which we next describe:

\paragraph{Our contribution} 

Our first result, which can be considered as a warmup problem, is for general analysts and not for GD. We show that if an analyst is allowed to query gradients of a convex function then $\Omega(1/\epsilon^3)$ samples are needed in order to provide $O(1/\epsilon^2)$, $\epsilon$--accurate answers. Our result here builds upon existing techniques and attacks \citep{bun2018fingerprinting, de2012lower, steinke2015interactive}, and we obtain a new lower bound (which may be of interest of its own right) to an analyst, in the standard statistical query setting, that queries many non-adaptive queries in sequential bulks of adaptive rounds (where the rounds of adaptivity are known, distinctively from \cite{dwork2015generalization}) and needs to obtain a fraction of true discoveries. We show that for such an analyst there exists a lower bound of $\Omega(\sqrt{T}/\epsilon^2)$ samples, where $T$ is the number of rounds of adaptivity and $\epsilon$ is the accuracy. We then show a generic reduction to the setting of convex optimization. Though the analyst is not GD, this result does demonstrate that one cannot design a complete mechanism for any analyst with optimal rates. It does leave open, though, the possibility of designing incomplete oracles that interact with specific types of analysts (or algorithms) such as GD.

The second result is for GD. We provide a bound of $\tilde \Omega(1/\epsilon^{2.5})$, but under further assumptions:
First, we assume the oracle is  \emph{post-hoc generalizing} \cite{cummings2016adaptive, ullman2018limits}. Roughly, posthoc generalization means that the algorithm does not query points where the empirical loss and true loss differ significantly. This assumption is restrictive, as many natural algorithm do not abide to it. However,
 we point out that we inherit it from existing known bounds in the standard statistical query setting. More accurately, then, our result can be rephrased, in this context, as a generic reduction. We apply the lower bound of \citet{ullman2018limits} that assumes post-hoc generalization. But more generally. given a lower bound for statistical queries of the form $f(T,\epsilon)$ where $T$ is the number of queries and $\epsilon$ is the desired accuracy, we provide a lower bound of the form $f(O(1/\epsilon),O(\epsilon))$ in the setting of convex optimization (under a further first-order access assumption which we discuss next).

The second assumption we make is what we term \emph{first-order access}. Here we assume that the oracle must compute the estimate only from the gradients at $\{w_1,\ldots w_t\}$ and not, say, by using the global structure of the function (we mention that our result can easily be extended to allow any local, but at a small neighbourhood, information of the function). Note that, since the function must be fixed throughout the optimization process, and since the optimization algorithm is fixed, allowing the oracle global access to the function restricts us from using any type of randomness other than the randomness of the distribution. Hence, while slightly more delicate then the first assumption, here too we require this assumption since in the standard statistical query setting lower bounds are provided with respect to random analysts. Our reduction, then, can turn a more general oracle (without first-order access) into a procedure that can answer statistical queries against a deterministic analyst (in the sense that the distribution may be random, but the analyst's strategy is fixed and known). This seems like an interesting question for future study.

It is interesting, then, to compare these results to recent adaptations of the standard model that restrict the analyst \cite{zrnic2019natural}. This is largely motivated by the reasoning that analysts are not necessarily adversarial. Our result, though, may hint (if one considers GD as a non malicious algorithm in this context) that the problem may be in the distribution of the data and not necessarily in the analyst. Namely, a general reduction from statistical queries to the framework of GD along our lines, will show that any lower bound can be described as constructing a malicious distribution which leads to overfitting together with a non-malicious analyst.

\section{Background}
\subsection{Adaptive Data Analysis}
We begin by revisiting the standard statistical queries setting of adaptive data analysis introduced by \citet{dwork2015preserving}. In this model, we consider a subset $\Q$ of statistical queries over a domain $\X$. A statistical query is defined to be any function $q:\X\to [-1,1]$. We consider a sequential interaction between an \emph{analyst} $A$ and a statistical queries \emph{oracle} $\O$ (or simply oracle) that continues for $T$ iterations and is depicted as follows:

At the beginning of the interaction the analyst $A$ chooses a distribution $D$ over $\X$. The Oracle $\O$ is provided with a finite collection of samples $S=\{x_1,\ldots, x_m\}$ drawn i.i.d from the distribution $D$. Then 
the interaction continues for $T$ sequential rounds: At round $t\ge 1$, $A$ provides a statistical query $q_t\in \Q$, and the oracle $\O$ returns an answer $a_t\in[-1,1]$. The answer $a_t$ may depend on the dataset $S$ as well as on previous answers and queries $\{q_1,\ldots, q_t\}$. The query $q_t$ may depend on previous answers $\{a_1\ldots, a_{t-1}\}$, as well as on the distribution $D$ (which is thought of as known to the analyst). We denote by $q_t(D)$ and $q_t(S)$ the following quantities:
\[ q_t(D) := \E_{x\sim D} [q_t(x)], \quad q_t(S) := \frac{1}{m}\sum_{i=1}^m q_t(x_i).\]

The goal of the oracle is to preserve accuracy, as next defined. And, here, we mostly care about the minimal size $m$ that is required by $\O$ in order to succeed.

\begin{definition}\label{def:oacc}
An oracle $\O$ is $(\epsilon,\gamma,\delta)$-accurate for $T$ adaptively chosen queries given $m$ samples in $\X$ if for every analyst $A$ and distribution $D$, with probability at least $(1-\delta)$ for $(1-\gamma)T$ fraction of the queries output by $A$:
\[|a_t-q_t(D)|\le \epsilon.\]
\end{definition}
We will write, for brevity, $(\epsilon,\delta)$-accurate instead of $(\epsilon,0,\delta)$-accurate. An additional requirement of \emph{post-hoc generalization} \cite{cummings2016adaptive}, is also sometimes imposed:
\begin{definition}\label{def:posthoc}
An oracle $\O$ is $(\epsilon,\delta)$-post hoc generalizing for $T$ adaptive queries with $m$ samples if:  given $m$ samples, for every analyst $A$, with probability at least $(1-\delta)$: for all $t\in[T]$
\[ \left| q_t(S) - q_t(D)\right|< \epsilon.\]
\end{definition}
The following result bounds the sample complexity of a post-hoc generalizing oracle:
\begin{theorem}[Cor 3.2 \cite{ullman2018limits}]\label{thm:stemmer}
There exists an analyst $A$, such that for every Oracle $\O$, if $\O$ is $(\epsilon,0.005)$ post hoc-generalizing and $(\epsilon,0.005)$-accurate, given $m$ samples, for $T$ adaptively chosen queries by $A$, then \begin{equation} m=\Omega(\sqrt{T}/\epsilon^2).\end{equation}
\end{theorem}

\subsection{Stochastic Convex Optimization}
We next review the setting of stochastic convex optimization. In this model we consider a function $f(w,x):\mathbb{R}^d\times \X \to \mathbb{R}$, which is \emph{convex} and $O(1)$-Lipschitz in the parameter $w$, for every $x\in \X$. We also consider a distribution $D$ over $\X$ and we denote by $F$ the \emph{population risk}:
\[ F(w) = \E_{x\sim D}[f(w,x)].\]
The objective of an optimization algorithm $A$ (or analyst) is to calculate $w^\star$ such that:
 \[ F(w^\star) \le \min_{\|w\|\le 1} F(w) + \epsilon.\]
 In order to achieve this goal, we also assume an interaction with what we'll call here \emph{exact-first-order} oracle, $\O_{f,\nabla f} (w,x)$, for the function $f$ that, given $w$ and $x$ returns 
 \begin{equation}\label{eq:firstorder} \O_{f,\nabla f} (w,x)=(\O_{f}(w,x), \O_{\nabla f}(w,x) ):= (f(w,x), \nabla f(w,x)).\end{equation}

 \paragraph{Gradient Descent over the empirical risk}
 A very popular first-order approach to solve the above optimization problem is by performing Gradient-Descent over the \emph{empirical loss} . Here we perform a very simple update rule: At first, the algorithm initializes at $w_1=0$. Then, at each iteration $t$ the algorithm updates 
 \begin{equation}\label{eq:fullbatchgd} w_{t+1} = \Pi\left[w_t-  \frac{\eta}{m}\sum_{i=1}^m \nabla f(w_t,x_i)\right],\end{equation}

 where $\Pi$ is a projection on the unit ball and $\nabla f$ is provided by access to an exact-first-order oracle. The output of the procedure is then: \[w_S= \frac{1}{T} \sum_{i=1}^T w_i.\]

 This procedure can be considered as an algorithm that minimizes the empirical loss, where given a sample $S$, we define the empirical loss to be
 \[F_S(w) = \frac{1}{m}\sum_{i=1}^m f(w,x_i).\]
 It is well known (see for example, \cite{bubeck2015convex}) that ,given the above procedure:
 \[ F_S(w_S) \le \min_{\|w\|\le 1} F_S(w) + O\left(\eta + \frac{1}{\eta T}\right).\]
 In particular, a choice of $\eta = O(1/\sqrt{T})$ leads to an error of $O(1/\sqrt{T})$. But the output of the procedure can also be related to the population risk through the following upper bound:
 
 \begin{theorem*}[\cite{bassily2020stability}]
 Let $D$ be an unknown distribution, over $\X$ and suppose that $f(w,x)$ is $O(1)$ Lipschitz and convex with respect to $w\in \reals^d$. Let $S=\{x_1,\ldots, x_m\}$ be a sample drawn i.i.d from distribution $D$, and consider the update rule in \cref{eq:fullbatchgd}. Then for $w_S=\frac{1}{T}\sum_{t=1}^T w_t$
 \begin{equation}\label{thm:bas} \E_{S\sim D^m}[F(w_S)]\le \min_{\|w^\star\|\le 1}F(w^\star) + O\left(\eta \sqrt{T} +\frac{1}{\eta T}+ \frac{\eta T}{m}\right).
 \end{equation}
 \end{theorem*}
A choice of $T=O(m^2)$, $\eta=1/m^{3/2}$ leads to an error of $O(1/\sqrt{m})$ which is statistically optimal.
 As discussed, \citet{amir2021sgd} provided a matching lower bound for the number of iteration required to achieve $O(1/\sqrt{m})$ error.
 
 One could also ask whether the empirical estimates of the gradients also generalize. Namely, is the empirical mean of the gradients close to their true expectations throughout the procedure? A close examination of the construction used by \citet{amir2021sgd} shows that, without special care, the empirical estimate of the gradient fails to provide accurate  gradients,  even if we choose the learning rate and number of iteration to minimize \cref{thm:bas}. We provide a proof sketch in \cref{prf:amir}
 \begin{theorem}\label{thm:amir} Given a sample $x_1,\ldots, x_m$ of i.i.d samples, suppose we run GD over the empirical risk, as depicted in \cref{eq:fullbatchgd}. There exists a distribution over $\X$ and an $O(1)$ convex Lipschitz function. Such that if $S$ is a sample drawn i.i.d from the distribution $D$ of size $m$ and $w_t$ is defined as in \cref{eq:fullbatchgd} then for $t=2$, with probability $1/2$ over $w_2$:
 \[ \left\|\E_{S\sim D^m} \left[\frac{1}{m}\sum_{i=1}^m \nabla f(w_2,x_i)|w_2\right]-\nabla F(w_2)]\right\| = \Omega(1).\]
 \end{theorem}
As discussed, there is in fact a simpler example to the fact that optimizing the objective doesn't require accurate gradients which is SGD. It is known that when $T=m$ and $\eta= 1/\sqrt{T}$, the analyst optimizes the objective to the same accuracy of $O(1/\sqrt{m})$. Remarkably, this requires even less iterations and the gradient doesn't even presume to be accurate. Nevertheless, the analysis here rely on the fact that the gradient is an \emph{unbiased} estimate, where adaptivity is avoided since we use a fresh example at every round.

It is also worth mentioning that recently \citet{koren2022benign} showed that, in adaptive data analysis terminology, SGD is an example to a non post-hoc generalizing algorithm in the following sense: It can be shown that the output parameter $w_S$ provided by SGD may minimize the population loss, but there is a constant gap between the empirical and population loss at $w_S$.
 \section{Problem setup}
 
 We next describe our setting. We consider the problem of adaptiveness within the context of stochastic convex optimization. 
 We consider an interaction between an analyst $A$ and a first-order optimization oracle, $\O_F$. At the beginning of the interaction the analyst chooses a function $f$ and a distribution $D$. Then a sample $S=\{x_1,\ldots,x_m\}$ is drawn and provided to the oracle. The interaction then proceeds for $T$ rounds, where at round $t\in[T]$ the analyst queries for a point $w_t$, and the oracle returns $\O_F(w_t)\in \reals^d$. The query points $w_1,\ldots, w_T$ may depend on the distribution $D$ and the oracle answer $\O_F(w_t)$ may depend on the sample $S$, the function $f$, as well as on the sequence of previously seen $w_1,\ldots w_{t-1}$. 
 
 \paragraph{Gradient Descent} Within our framework we describe GD as the following procedure: For every $\eta>0$ we let GD with learning rate $\eta$ be defined by the following update at each iteration $t\ge 1$ (setting $w_1=0$):
 \begin{equation}\label{eq:gdo} w_t = \Pi\left(w_t- \eta \nabla \O_F(w_t)\right),\end{equation}
 where $\Pi$ is the projection operator over the $\ell_2$-unit ball.
Notice that, if $\O_F$ has an access to an exact first order oracle for $f$, $\O_{\nabla f}$, and returns the empirical mean at each iteration then we obtain GD over the empirical risk as described in \cref{eq:gdintro}.
Before we continue, we notice that good generalization of $\O_F$ is sufficient for optimization. Indeed, the following result is an easy adaptation of the classical optimization bound for GD:
 
 \begin{theorem}\label{thm:trivial}
  Let $D$ be an unknown distribution over $\X$ and suppose that $f(w,x)$ is $O(1)$-Lipschitz and convex with respect to $w\in \reals^d$. Let $S=\{x_1,\ldots, x_m\}$ be a sample drawn i.i.d from distribution $D$, and consider the update rule in \cref{eq:gdo}. Assume that for every iteration $\O_F$ satisfies
  \[ |\O_F(w_t)- \nabla F(w_t)|\le \epsilon,\]
 Then for $w_S=\frac{1}{T}\sum_{t=1}^T w_t$
 \[ E_{S\sim D^m}[F(w_S)]\le \min_{\|w^\star\|\le 1}F(w^\star)+ O\left(\eta +\frac{1}{\eta T}+ \epsilon\right).\]
 \end{theorem}
The above rate is optimal, which leads to the natural question, what is the sample needed  by an oracle that returns $O(1/\epsilon^2)$ $\epsilon$-accurate gradients for GD with learning rate $O(\epsilon)$. Such an oracle improves over, the naive, empirical mean estimate which induces GD over the empirical risk which requires $\tilde \Theta(1/\epsilon^4)$ iterations to achieve error of $O(\epsilon)$. The performance of such an oracle should also be compared with SGD that can achieve a comparable bound on the number of iterations and requires the optimal sample size of $m=O(1/\epsilon^2)$. 
Next, we provide natural extentions to the definition of adaptive oracles to the setting of stochastic optimization.
\begin{definition}\label{def:accF}
A first order oracle $\O_F$ is $(\epsilon,\gamma,\delta)$-accurate against algorithm $A$ for $T$ iterations, given $m$ samples, if $\O_F$ is provided with $m$ samples and with probability at least $(1-\delta)$ for $(1-\gamma)T$ fractions of the $t\in [T]$:
\[\|\O_F(w_t) - \nabla F(w_t)\| \le \epsilon.\]

If $\O$ is $(\epsilon,\gamma,\delta)$-accurate against any algorithm $A$ we say it is $(\epsilon,\gamma,\delta)$-accurate.
\end{definition}
We will write in short $(\epsilon,\delta)$-accurate for $(\epsilon,0,\delta)$-accurate.
\begin{definition}\label{def:posthocF}
A first-order oracle $\O_F$ is $(\epsilon,\delta)$-post hoc generalizing against algorithm $A$ for $T$ iterations, given $m$ samples if with probability at least $(1-\delta)$: for every $t\in[T]$
\[ \| \nabla F(w_t) - \frac{1}{m}\sum_{i=1}^m \nabla f(w_t,x_i) \|\le \epsilon.\]

If $\O$ is $(\epsilon, \delta)$-post hoc generalizing against any algorithm $A$ we simply say it is  $(\epsilon, \delta)$-post hoc generalizing.
\end{definition}
\paragraph{First order local access} We next introduce the following assumption on the oracle:
\begin{definition}
A first order first-order-access (\fla)-oracle $\O_F$ is a procedure that, given access to an exact-first-order oracle $\O_{f,\nabla f}$ to an $L$-Lipschitz function $f$, and access to a sample $S$ of size $m$, returns for every point $w_t$ a gradient estimate $\O_F(w_t)$ that may depend, at time $t$, only on \[\left\{(f(w_{t'},x_i), \nabla f(w_{t'},x_i\} \right\}_{\{(x_i, w_{t'}): x_i\in S, t'\le t\}}.\]

\end{definition}

Equivalently, we may think of an \fla oracle as a procedure that does not have access to $f$, instead, at each iteration $t$ receives as input the parameter $w_t$ as well as a gradient-access function \[\bar \rho_t: \X \to \reals\times \reals^d,\] such that \begin{equation}\label{eq:rho0rho1} \bar \rho_t(x) := (\bar \rho^0_t(x),\bar \rho^1_t(x)) :=  (f(w_t,x),\nabla f(w_t,x)),\end{equation} for every $x$. The output of the \fla at round $t$ may depend on $\bar \rho_1,\ldots, \bar \rho_t$
The assumption of a \fla-oracle is very natural in the context of Stochastic Convex Optimization, and in general, we do not assume access to a global structure of a convex function. The above assumption indeed captures oracles that have only such local access.

\section{Main Results}
We are now ready to state our main results. Our first result state that, for a general analyst, the oracle cannot generalize for $T=O(1/\epsilon^2)$ estimated gradients, unless it is provided with $m=\Omega(1/\epsilon^3)$ examples. The proof is provided in \cref{sec:prf:main1}
\begin{theorem}\label{thm:main1}
There exists constants $\gamma,\delta>0$ and a randomized analyst $A$ the chooses a determined $1$-Lipschitz function $f$, defined over sufficiently large $d$, such if $\O_F$ is a first-order oracle that is $(\epsilon,\gamma,\delta)$-accurate against $A$ for $T$ iterations, then $m= \Omega\left(\frac{\sqrt{T}}{\epsilon^2}\right)$. In particular, any oracle $\O_F$ that is $(\epsilon,\gamma,\delta)$-accurate for $T=O(1/\epsilon^2)$ iterations must observe $m=\Omega\left(1/\epsilon^3\right)$ examples.
\end{theorem}

Making no assumption on the analyst may seem non-realistic, especially to assume it is malicious and attempts to achieve false gradients. Nevertheless there is value in producing oracles that are foolproof. The above theorem shows that such security guarantees are impossible with the standard sample complexity.

The next natural thing that we might want to consider is an oracle that is principled under certain assumptions on the optimization algorithm. We might even hope to design an incomplete oracle that can interact with specific optimization algorithms and halt when certain assumptions are broken. The next result demonstrate that limitations from general statistical queries can be translated into limitations for (\fla) oracles that interact with gradient descent. The proof is provided in \cref{sec:prfmain2}.

\begin{theorem}\label{thm:main2}
For sufficiently large $d$, suppose that there exists a  \fla oracle, $\O_F$,  that is $(\epsilon,\delta)$-accurate that receives $m$ samples and answers $T$ adaptive queries against Gradient Descent with learning rate $\eta=O(\epsilon)$. Then there exists a  $(O(\epsilon),O(\delta))$-accurate statistical queries oracle, $\O$, that receives $m$ samples and answers $\tilde\Omega\left(\min\{T, 1/\eta\}\right)$ adaptive queries.

Moreover, if $\O_F$ is  $(\epsilon,\delta)$-post-hoc generalizing then $\O$ is $(O(\epsilon),O(\delta))$-post-hoc generalizing.
\end{theorem}
Together with \cref{thm:stemmer} we obtain the following corollary
\begin{corollary}\label{cor:main2}
For sufficiently large $d$, let $\O_F$ be an $(\epsilon,\delta)$-accurate and post-hoc generalizing \fla that receives $m$ samples and answers $T>\Omega(1/\epsilon)$ adaptive queries against Gradient Descent with learning rate $\eta=O(\epsilon)$. Then $m= \tilde\Omega\left(1/\epsilon^{2.5}\right)$.
\end{corollary}

We stress again that, in contrast with these results,  an optimization algorithm can correctly minimize the true loss using no more than $\tilde O(1/\epsilon^2)$ iterations and $\tilde O(1/\epsilon^2)$ examples \cite{koren2022benign}. 

\section{Lower bounds against malicious analysts}\label{sec:main1}
In this section we set out to prove \cref{thm:main1} and provide a lower bound to general oracles against \emph{adversarial} analysts. In section \cref{sec:convexreduction} we show how to turn a generic lower bound for statistical queries to a lower bound in convex optimization. However, to the best of the author's knoweledge, there is no (unconditional) known lower bound that shows that $m=\Omega(\sqrt{T}/\epsilon^2)$ examples are necessary to answer $T$ queries. So we actually rely on a slightly stronger reduction than from the standard setting of statistical queries. %
We rely, then, on the fact that querying a single gradient carries more information than a single statistical query, in fact $d$ more. However, these may not be chosen adaptively, and the errors are spreaded. 

The setting from which we provide the reduction is as follows: We consider an analyst that at each iteration $t$ can ask $k$ non-adaptive queries. This is reminiscent to a similar problem that was studied by \citet{dwork2015preserving}, but there it is unknown what are the rounds of adaptivity. Here we consider a significantly simpler problem where the rounds of adaptivity are known in advance and we show, using ideas from \cite{bun2018fingerprinting} (that constructs a similar lower bound but in the setting of privacy),  that for certain $k=\Omega (1/\epsilon^2)$,$\Omega(\sqrt{T}/\epsilon^2)$ samples are needed to ensure a large enough fraction of the answers are correct. Then, as discussed, we provide a generic reduction to convex optimization. We now turn to describe the setting of adaptive non-adaptive queries and state our main lower bound for this setting. 

\subsection{Adaptive-non-Adaptive queries}
In this section we take a little detour from our basic setting and return to the setting of statistical queries.
\subsubsection{Setup} We will consider now a natural generalization of the standard setting of adaptive data analysis. Here,  we allow the analyst to query at each round $k$ queries simultaneously.
In this setting, as before, we have a family of queries $\Q$ as well as an analyst $A$ and oracle $\O$ which interact for $T$ rounds. Distinictively from before, at round $t$ we assume $A$ asks $k$-statistical queries $\q_t=\{q_{t,1},\ldots, q_{t,k}\}\subseteq \Q^k$, and $\O$ returns an answer vector $\a_t=(a_{t,1},\ldots, a_{t,k})$. The answer vector $\a_t$ may depend on the sample $S$ and on previously published queries $\q_1,\ldots,\q_t$. Similarly the query vector may depend on previous answer vectors $\a_1,\ldots, \a_{t-1}$ and the distribution $D$.

\begin{definition}\label{def:oacck}
Similar to \cref{def:oacc}, we say that $\O$ is $(\epsilon,\gamma_T,\gamma_k,\delta)$-accurate for $T$ adaptively chosen queries, given $m$ samples, if the oracle samples at most $m$ samples and with probability $1-\delta$ we have for $(1-\gamma_T)$ fraction of the rounds, for $(1-\gamma_k)$ fraction of the queries: 
\[|\O(q_{t,i})-q_{t,i}(D)|\le \epsilon.\] 
\end{definition}
We next set out to prove the following lower bound:
\begin{theorem}\label{thm:kT}
For $\X= \{0,1\}^{}$
For $k= \Omega(1/\epsilon^{2.01})$, there exists a finite family of queries $\Q$ over the domain $\X=\{0,1\}^k$, constants $\gamma_T,\gamma_k$, $\delta$, such that no oracle $\O$ is $(\epsilon,\gamma_T,\gamma_k,\delta)$ accurate for $k$-non adaptive $T$ adaptively chosen queries given $m$ samples unless $m=\Omega \left(\sqrt{T}/\epsilon^2\right)$
\end{theorem}

Before we begin with the proof, we provide several preliminary results that we build upon.

\subsubsection{Overview and technical preliminaries}
The proof of \cref{thm:kT} relies on a technical idea that appears in \cite{bun2018fingerprinting}. \citet{bun2018fingerprinting} starts by considering two constructions in the context of privacy. The first, demonstrates a sample complexity lower bound of $\Omega(\sqrt{T})$ for $T$ private queries, and a second construction, a reconstruction attack, that allows a certain reconstruction of the data unless the sample size is order of $\Omega(1/\epsilon^2)$ for $\epsilon$-accurate answers. Then, they provide a new construction that consolidates these two bounds into one construction that operates on a certain product space of the two domains.
Here we do something similar only we replace the privacy attack with an adaptive data analysis attacks that operates on i.i.d samples (which is not necessary when privacy is considered). The consolidation is a little bit different as we must consider a dataset that is generated by sampling i.i.d examples (as opposed to worst-case dataset in the case of privacy)                          .

In more detail, the proof of \cref{thm:kT} relies on two types of attacks that were introduced by \citet{steinke2015interactive, de2012lower}. Our first type of attack is a reconstruction attack, and we follow the definition of \citet{bun2018fingerprinting}:
\begin{definition}[Reconstruction Attack]
For a dataset $\S=\{x_1,\ldots,x_m\}$, we will say that $\S$ enables an $\epsilon'$-reconstruction attack from $(\epsilon,\gamma)$-accurate answers to the family of statistical queries $\Q$ if: There exists a function
\[\B : \reals^{|Q|}\to [0,1]^m,\]
such that for every vector $v\in [0,1]^m$ and every answer sequence $a=(a_q)_{q\in \Q}\in [0,1]^\Q$: If for at least $1-\gamma$ fraction of the queries $q\in \Q$ holds:
\[\left|a_q-\frac{1}{m}\sum_{i=1}^m q(x_i)v(i)\right| < \epsilon,\]
then for $\b=\B(a)$:
\[ \frac{1}{m}\sum_{i=1}^m |b(i)-v(i)| < \epsilon'.\]
\end{definition}
The following result is due to \citet{de2012lower}, we state it as in \cite{bun2018fingerprinting} for the special case of considering $1$-way marginals\footnote{Note that in \cite{bun2018fingerprinting} $k$ denotes the $k$-way marginal query class which we fix to be the $1$-way marginal, and the $k$ in our statement is denoted by $d$ in \cite{bun2018fingerprinting}} :
\begin{theorem}\label{thm:reconstruction}
Let $k\ge 1/\epsilon^{2.01}$, and assume $\epsilon$ is sufficiently small. There exists a constant $\gamma_0$ (independent of $\epsilon$ and $k$) such that for every $\epsilon'$, there exists a dataset $\S= (\{0,1\}^k)^m$ with $m=\Omega_{\epsilon'}(1/\epsilon^2)$ such that $\S$ enables an $\epsilon'$-reconstruction attack from $(\epsilon,\gamma_0)$-accurate answers to a family of queries $\Q$ of size $k$.
\end{theorem}

The second attack that we rely on provides an information theoretic lower bound of $\Omega(\sqrt{T})$ to answer adaptive statistical queries:
\begin{theorem}[Thm 3.10 \cite{steinke2015interactive}]\label{thm:interactive0}
For all $\gamma<1/2$, there is a function $T(m,\gamma)\in O\left(m^2/{(1/2-\gamma)^4}\right)$, such that there is no oracle $\O$ that is $(0.99,\gamma,1/2)$-accurate  for $T(m,\gamma)$ adaptively chosen queries, given $m$ samples in $\{0,1\}^d$, where $d\ge T(m,\gamma)$. 
\end{theorem}

We will require a dual restatement of \cref{thm:interactive0}, which essentially follows the same proof together with standard minmax theorem:
\begin{theorem}\label{thm:interactive}
There exists a randomized analyst $A$ such that for any oracle $\O$ that interacts with $A$ for $T(m,\gamma)$ rounds having $m$ samples, then with probability at least $1/2$ for at least $\gamma T$ of the rounds: 
\[ |a_t - q_t(D)|>0.01.\]
\end{theorem}
\begin{proof}[Sketch]
The proof is essentialy the proof of \cref{thm:interactive0} as depicted by \citet{steinke2015interactive}. We only need to argue that in the construction of \citet{steinke2015interactive} the advarsarial analysts that are being constructed are from a finite set and then use standard minmax duality.
To see that the analysts in the original proof are supported on a finite set, first observe that the analyst chooses (randomly) a uniform distribution over a sequence of pairs $(1,v_1),(2,v_2),\ldots, (N,v_N)$ where $N=T(m,\gamma)$ and each $v_i$ depicts a secret key (where for the information-theoretic lower bound we choose a one-time pad encryption and then $v_i\in \{\pm 1\}^N$). Hence the set of feasible distributions is of size $2^N$, and $N=O(T(m,\gamma))$.
Next, we note that at each iteration, the analyst rounds the answer for $\O$, $a_t$ and chooses as a query, $q_{t}$ which is parameterized by a vector in $\{-1,0,1\}^N$. Hence, the query at round $t$ depends on $t$ vectors in $\{-1,0,1\}^N$ and $\{\mathrm{sgn}(a_1),\ldots,\mathrm{sgn}(a_t)\}$ hence overall there is a finite set of states to which the analyst can transition at each iteration, so overall there is only a finite amount of analysts on which the distribution is supported.
\end{proof}

\subsection{Proof of \cref{thm:kT}}\label{sec:convexreduction} 

Let $k\ge 1/\epsilon^{2.01}$, and set $\Q_{\epsilon}$ be a set of at most $k$ queries over a dataset $\S_\epsilon$ and $\gamma_0$ a parameter that enables an $\epsilon'$-reconstruction attack from $(\epsilon,\gamma_0)$-accurate answers to $\Q_\epsilon$ as guaranteed in \cref{thm:reconstruction}. Without loss of generality we assume $\gamma_0<1/2$, and $\epsilon'$ is chosen such that:
\[\epsilon' <\frac{0.01}{3\cdot 2^6}.\]
Now we let $d= |S_\epsilon|=O(1/\epsilon^2)$. Without loss of generality and for simplicity of notations we assume $S_{\epsilon}=[d]=\{1,\ldots, d\}$.

Suppose $\Q$ is a family of queries, and assume we have $d$ analysts, in the standard model (i.e. each analyst asks a single question), $(A_{1},\ldots,A_{d})$. We define an analyst that asks $k$ queries $\A (A_1,\ldots,A_d)$ as follows: 

First, when the analysts choose distributions $D_{1},\ldots, D_{d}$ over $\X$, the analyst $\A$ defines a distribution $D$ over $\S_{\epsilon}\times \X$ that chooses first randomly and uniformly $i\in [d]$ and returns $(i,x)$ where $x\sim D_i$. The oracle, in turn, observes i.i.d samples from the given product distribution.

The interplay with the oracle proceeds as follows:  At each iteration $t$, we assume by induction that each analyst, $A_1,\ldots,A_d$, provides a query $q_{t,1},\ldots, q_{t,d}$. The analyst $\A(A_1,\ldots, A_d)$ constructs for each query $q\in \Q_{\epsilon}$ the query
\[q_t'((i,x))=q(i)q_{t,i}(x),\]
and asks these $k$ non-adaptive queries.

Then, given the answer vector $\{\a_{q_t'}\}$, we provide analyst $A_i$ with the answer $a_{t,i}$, where
\[\a_t= \B(\a_{q_t'}),\]
and $\B$ defines the reconstruction attack in \cref{thm:reconstruction}. The analysts then provide the queries $q_{t+1,1},\ldots, q_{{t+1},d}$ and the analyst $\A(A_1,\ldots, A_d)$ continues to the next round until round T.
Our analyst then depends on the $d$ analysts, We choose them to be $d$ i.i.d copies of the analyst in \cref{thm:interactive} and we let $\bar A$ be the analyst induced by such $A_1,\ldots, A_d$. 

Notice that when we fix $A_1,\ldots, A_{i-1},A_{i+1},\ldots, A_d$, that are provided to $\A(A_1,\ldots, 
A_d)$ we induce an oracle, that we denote by $\O_i$ that interacts with analyst $A_i$. 
In more detail, we consider a randomized oracle $\O_i$ that operates as follows:

At the beginning of the interaction, before the first round, $\O_i$ draw a uniform sample $\{s_1,\ldots, s_m\}$. For each sample $s_j\ne i$ the oracle also draws a sample $(x\sim D_{i})$. Then, given $m_i$ samples $\{x_1,\ldots, x_{m_i}\}$ from $D_{i}$ where $m_i$ is the number of times $i$ was drawn, the oracle adds to the sample the sample points $\{(i,x_j)\}_{j=1}^{m_i}$. Notice that this sample is drawn exactly according to the process depicted above where $(i,x)$ is drawn such that $i$ is uniform and $x\sim D_i$. The interaction with $A_i$ along the rounds is continued where at each round $\bar A$ transmit  the question, $\O$ answers, and $\bar A$ transmit the answer back, as described above.

In this interaction the number of samples is random, but notice that if $T \ge T(m_i,1/4)$, where
\[ T(m,1/4)= O(m^2),\]
is defined in \cref{thm:interactive}, 
then with probability at least $1/2$ for $T/4$ of the rounds, by \cref{thm:interactive}
\[ \|a_{t,i}-q_{t,i}(D_i)\| > 0.01.\]
 This also entails that for every $i$:

\begin{equation}\label{eq:atiqti} \E_{A_i,Q_i}\left(\frac{1}{T}\sum_{t=1}^T \left|a_{t,i}-q_{t,i}(D_i)\right|\right) > \frac{0.01}{2^3} \mathbb{P}(T\ge T(m_i,1/4)).\end{equation}

Now assume $\O$ is $(\epsilon,\epsilon',\gamma_0,\epsilon')$-accurate for $T=T(4m/3d,1/4)$ rounds, and consider $\A$ as defined above.
Then with probability $1-\epsilon'$: for $(1-\epsilon')$-fractions of the rounds, for $(1-\gamma_0)$-fraction of the queries $q_t'$:

\[
\left|a_{q_t'}-\frac{1}{k}\sum_{i=1}^k q(i)q_{t,i}(D_i)\right|=\left|a_{q'_t}-q_t'(D)\right|<\epsilon.\]
Which entails by reconstruction attack, for the same fraction of rounds:
\[ \frac{1}{k}  \sum_{i=1}^k |a_{t,i}-q_{t,i}(D_i)|\le \epsilon'.\]
Taken together we have 
\begin{equation}\label{eq:ub} \E\left[\frac{1}{Tk}  \sum_{i=1}^k |a_{t,i}-q_{t,i}(D_i)|\right]\le 3\epsilon'\le \frac{0.01}{2^6}.\end{equation}

On the other hand, notice that for any analyst, with probability $1/4$, $\O_i$ is provided with less than $\frac{4m}{3d}$  samples from the distribution $D_i$ 

So by choice $T\ge T(4m/3d,1/4)=O(m^2\epsilon^4)$, 
and by \cref{eq:atiqti}, we have:

\[\E_{A,Q}\left( \frac{1}{Tk}\sum_{i=1}^T\sum_{i=1}^k |a_{t,i}-q_{t,i}|\right) =\frac{1}{k} \sum_{i=1}^k \E_{A_i,Q_i}\left(\frac{1}{T}\sum_{t=1}^T|a_{t,i}-q_{t,i}|\right)> \frac{0.01}{2^5}.\]

contradicting \cref{eq:ub}.
\subsubsection{Proof of \cref{thm:main1}}\label{sec:prf:main1}

We now proceed with the formal proof of \cref{thm:main1}. Given a family of queries $|\Q|\le d$ index the coordinates of $\reals^d$ by the elements of $\Q$. Namely, we think of $\reals^d$ as $\reals^\Q$ where each vector $w\in \reals^Q$ is thought of as a function $w:\Q\to \reals$.

We define a convex, over the parameter $w$, function in $\reals^\Q$:
\begin{equation}\label{eq:poc} f(w,x)= \sum_{q\in \Q} \frac{q(x)+1}{4} w^2(q).\end{equation}
Note that, since $|q(x)|\le 1$, the above function is always convex and $1$-Lipschitz for any choice of queries and $x$.
Let $\O_F$ be a first order $(\epsilon,\gamma_T,\delta)$-accurate oracle, and let us consider the setting of an analyst that requires $k$ non adaptive queries for $T$ rounds. 
Let $A$ be an analyst that asks $k$ non adaptive, $T$ adaptive queries, and consider an oracle $\O$ that performs as follows: Given queries $q_{t,1},\ldots, q_{t,k}$, the oracle $\O$ transmit to the oracle $\O_F$ the point:
 \[w_t= \frac{1}{\sqrt{k}}\sum_{i=1}^k e_{q_{t,i}},\]
 where $e_q$ is the vector in $\reals^Q$ that is $e_q(q)=1$, and $e_q(q')=0$ when $q'\ne q$. In turn, the oracle receives the answer vector $\O_F(w_t) =g_t$ and returns the answers $a_{t,i}= 2\sqrt{k} g_t(q_{t,i})-1$
 
Now suppose that with probability $(1-\delta)$ for $\gamma_T$ fraction of the rounds:
\[\|g_t-\nabla F(w_t)\|\le \epsilon,\]
then:
\[
\frac{1}{k}\sum \left(\frac{a_{t,i}-q_{t,i}(D)}{2}\right)^2
=
\sum \left(\frac{a_{t,i}+1}{2\sqrt{k}} - \frac{q_{t,i}(D)+1}{2\sqrt{k}} \right)^2
\le  \|g_t -\nabla F(w_t)\|^2\le \epsilon^2.\]
Then by Markov's inequality for any $\gamma_k$, for $\gamma_k$ fraction of the queries we have:
\[ |a_{t,i}-q_{t,i}(\D)|\le \frac{2\epsilon}{\sqrt{\gamma_k}}.\]
Which means $\O$ is an $(\frac{2\epsilon}{\sqrt{\gamma_k}},\gamma_T,\gamma_k,\delta)$-accurate oracle that answers $k$ non adaptive $T$ adaptive queries. By \cref{thm:kT}, for small enough $\epsilon$, with correct choice of $\gamma_k$, $\gamma_T$ and $\delta$:
\[m=\Omega\left(\sqrt{T}/\epsilon^2\right).\]

\section{Gradient Descent}
In this section we set out to prove \cref{thm:main2}. In contrast to previous result, here we fix the analyst and assume that it performs predefined update steps. This puts several complications into the lower bound as we cannot actively make it ``adversarial", at least not in the standard way. Nevertheless our construction builds on a similar idea as the proof before. The idea here is to think of the function as a ``state" machine, where each coordinate represents a query that may be asked. The analyst, given answers to the queries, moves to the next query. The complication though, is that here the dynamic is predefined and we need to design our function carefully so that GD will induce the correct transition between states. 

The idea is captured in what is our main technical workhorse which is the notion of a \emph{GD wrapper}, which we build in \cref{sec:construction}. GD wrappers will be used to provide a reduction from a special class of analysts termed \emph{Boolean analysts}, which are depicted in \cref{sec:overview}. Then we use a simple reduction from general analysts (see \cref{lem:boolean}) to obtain a reduction from general analysts to our setting.

We begin with a brief overview of the construction. After that, in \cref{sec:convex_wrapper} we depict the technical notion of GD wrapper. We then explain, in \cref{sec:prfmain2}, how to deduce \cref{thm:main2} from the existence of a GD wrapper. Finally, in \cref{sec:construction} we provide a construction of a GD wrapper, which concludes the proof.

\subsection{Brief overview}\label{sec:overview}
As discussed, the heart of our construction is the notion of a \emph{GD wrapper}. The idea is quite straightforward. Given an analyst $A$ we want to construct a convex function $f_A$ that is convex and such that the trajectory of the function provides the answers to our question.
We've done something similar in the construction of \cref{eq:poc}. There too, we constructed a convex function that the gradient of $w_t$ at a certain coordinate provides an answer to a given query. The issue though is that there we could ask to query any coordinate we wanted. Here we need to make sure that the dynamic of GD moves us from query $q_t$ to query $q_{t+1}$. 

Thus, the first requirement that we want from our function $f_A$ is that by looking at the outputs: $\O(w_1),\ldots, \O(w_{t-1})$ we can identify the answer to query $q_t$. For one non-adaptive query this is straight forward. Indeed, consider the linear function
\[f_q(w,x) = q(x)\cdot w(1).\]
Then, the gradient $\nabla F_q(w)= \E_{x\sim D}[q(x)]e_1$. So we can identify the answer to the first query. 

\paragraph{$2$ Queries:} As a next step, let us construct a convex function where GD works as follows: At the first step the function will provide answer to query $q_1$, and if $q_1(x)>b_1$ for some threshold $b_1$, then the function transitions to a state $w_2$ that identifies the answer to a query $q^{+}$ and else moves to a state that identifies $q^{-}$ (to simplify, we will assume that the answer is never $q = b_1$). This is still far from a general strategy of an analyst, but at the end of this overview we will discuss how we can reduce the general problem to a problem of a similar form.

Also, for the exposition we don't want to consider the oracle's strategy, hence assume that at each iteration the oracle returns the \emph{true gradient} and we will show how the trajectory simulates the adaptive query interaction:

\[f_{q,q^+,q^-} = \max\left\{ w(1)+ \eta - \frac{1-q(x)}{3}w(2) - \frac{1+q(x)}{3}w(3), w(2) +q^+(x)w(4), w(3) + q^{-}(x)w(5)\right\}.\]
Our function is described as the maximum of three linear functions hence it is convex. Now let us follow the trajectory for the first two steps.
At the first step, note that the first term maximizes the term. Recall that the gradient of a function $f=\max\{g_1,\ldots,g_k\}$, is given by $\nabla f=\argmax \nabla_{i\in [k]} g_i$ hence we have that for $w_1=0$, for every $x$:
\[ \nabla f_{q,q^+,q^-}(0,x) = e_1 - \frac{(1-q(x)}{3}e_2 - \frac{1+q(x)}{3}e_3.\] hence:
\[w_2= w_1-\eta \nabla \E_{x\sim D}[f_{q,q^+,q^-}(0,x)] = -\eta e_1 + \frac{(1-q(D))}{3}\eta e_2 + \frac{1+q(D)}{3}\eta e_3.\]

Now, note that the first term is no longer maximized by $w_2$ for any $x$, as we moved against the gradient and now it is smaller. On the other hand, if $q(D)<0$ the second term is maximized, and else the last. Assume the first: then
\[ \E_{x\sim D}\nabla f(w_2,x) = e_2 + q^{+}(D) e_4,\]
Note that the gradient at $w_2$ tells us excatly whether $q(D)>0$ or $q(D)<0$. In particular, if $\O(w_2)_2>1/2$, then we know the $q(D)>0$. Any oracle that returns an approximate answer will also identify the answer. Using a recursive process, along these lines, we then can construct a convex function that moves from one query to another using gradient descent.

In general, the state of an analyst does not depend, necessarily, on some threshold value as in our case above. However, as the next reduction shows, if we are willing to suffer a $\log 1/\epsilon$ factor increase in the number of queries, we can turn a general analyst to an analyst whose decision indeed depend on some threshold as depicted here:

\paragraph{Boolean Analysts}
We will call an analyst \emph{Boolean} if it provides to the oracle a query $q_i:\X\to[-1,1]$
and its state at time $t$ depends only on $\{\mathrm{sgn}(a_1),\mathrm{sgn}(a_2),\ldots, \mathrm{sgn}(a_{t-1})\}$. A more general setup could allow at each iteration a query $q_i$ and a threshold $b_i$ and the state may depend only on $\mathrm{sgn}(a_{1}-b_1),\ldots, \mathrm{sgn}(a_{t-1}-b_{t-1})$. However, up to rescaling it can be seen that the two types of analysts are equivalent.

For such a boolean analyst, we define an oracle to be $(\epsilon,\delta)$-accurate for $T$ adaptive queries against a Boolean analyst, if given $m$ samples for every Boolean analyst $A$ with probability at least $(1-\delta)$ for all $t\in [T]$ if $a_t=1$ then 
\[ q_t(D) > -\epsilon,\]
and if $a_t=-1$
\[ q_t(D) < \epsilon.\]

Similarly, an oracle $\O$ is $(\epsilon,\delta)$-post hoc generalizing for $T$ adaptive queries against a Boolean analyst, if given $m$ samples for every Boolean analyst $A$ with probability at least $(1-\delta)$ for all $t\in [T]$ if $q(D)>0$ then 
\[ \frac{1}{m}\sum q_t(x_i) > -\epsilon,\]
and if $q(D)<0$
\[ \frac{1}{m}\sum q_t(x_i) < \epsilon.\]
The following statement is easy to see:
\begin{lemma}\label{lem:boolean}
Suppose that there exists an oracle $\O_1$ that is $(\epsilon,\delta)$-post hoc generalizing $(\epsilon,\delta)$-accurate oracle against any Boolean analyst that answers $T$ queries with $m$ samples. Then there exists a $(4\epsilon,\delta)$-post-hoc generalizing $(4\epsilon,\delta)$-accurate oracle, $\O_2$, that answers any analyst (not necessarily Boolean) $\Omega\left(\frac{T}{\log1/\epsilon}\right)$ queries with $m$ samples.
\end{lemma}

\paragraph{General first order local access oracles} We so far assume a truthful Oracle, that provides the true answer. We now need to deal with an oracle whose decision is both based on finite data and is \emph{strategic}, in the sense that it can manipulate the wrapper above with his answers. 

Note that in order to construct the wrapper we designed a function whose dependent on \emph{all} possible states of the analyst $A$. In particular, an all powerful oracle can basically look at the design of the function and get to know the exact strategy. That includes random bits as the function needs to be determined and chosen at the beginning of the game. That is why, as long as existing lower bounds for statistical queries rely on randomized analysts, for this strategy to work, we need to somehow prohibit from the oracle to identify the random bits. Note that the reverse is also true. Without some restrictions on the oracle, the construction against Gradient Descent becomes a pure strategy (modulus the choice of the distribution of the data).

We thus assume that the oracle has only access through the local gradients at points visited which restricts him from identifying the whole strategy of the analyst.
In turn, we need to make sure, in our construction, that such internal bits are indeed not transmitted through the gradients. We therefore add in our construction a further random embedding that hides this further information, given only past information of gradients at point visited (which is a restriction on the \fla oracle).
 
\subsection{\fla-GD wrapper for a data analyst}\label{sec:convex_wrapper}

The main workhorse in our construction is what we will term here an \fla-GD wrapper (or just wrapper). In a nutshell, the wrapper can be thought of as an object that allows an analyst (in the statistical query setting) to interact with an \fla Oracle. More formally, an \fla-GD wrapper for a data analyst (with learning rate $\eta$ and initialization $s$) consists of
\begin{enumerate}
\item A \emph{wrapper function} which is a function $f(A;w,x)$ that accepts a deterministic Boolean analyst and for every analyst $A$ it is convex and $1$-Lipschitz in a parameter $w\in \reals^d$.
\item A strictly increasing mapping $\T:[T_1]\to [T_2]$, $\T(1)>1$. The wrapper is said to answer $T_1$ queries and to perform $T_2$ iterations, and $\T$ is called the iteration complexity.

\item A sequence $\vec{\kappa}= \{\kappa_t\}_{t=1}^{T_1}$ of $T_1$ functions which are termed \emph{answering mechanisms}:
\[ \kappa_{t}: \left(\reals^d\right)^{\T(t)} \to [0,1],\]

\ignore{\item A sequence $\vec{ \rho} = \{\rho_t\}_{t=1}^{T_2}$ of $T_2$ functions which are termed \emph{gradient access functions}:
\[ \rho_t : \Q \times \left(\reals^d\right)^t\times \X\to \reals  \times \reals^d,\]
where $\T(t')<t$.}

\end{enumerate}

The GD wrapper interacts with an \fla oracle as follows.
Given a fixed \fla Oracle $\O_F$, at step $t$ we define inductively a point $w_t$ (where $w_1=s$), and at round $t$ the $t$-th gradient access function is given by \[\bar \rho_t(x)=(f(A,w_t,x), \nabla f(A,w_t,x)).\]
$\bar\rho_t$ is provided to the oracle $\O_F$, which in turn yields $\O(w_t)= g_t$. We then define
\[ w_{t+1}=\Pi\left(w_{t}-\eta g_{t}\right).\]
We will call this sequence the trajectory of the GD wrapper against $\O_F$
Finally, we define the answering sequence, which is updated whenever $t=\T(t')$ for some $t'$:
 \[a_{t'} = \kappa_{t'}(g_1,\ldots, g_t).\]

The GD wrapper is said to be $(\epsilon_1,\epsilon_2,\delta)$-accurate against $\O_F$ if for every distribution $D$,   the following occurs w.p. $(1-\delta)$, for \emph{every} $t_0\le T_1$:

\paragraph{Accuracy of gradients entails accuracy of answers:}
If, for analyst $A$:
 \begin{equation}\label{eq:acc} \|g_t- \E_{x\sim D}[\nabla f(A;w_{t},x)]\|\le \epsilon_1.\end{equation}
for every $t\le  t_0$, and $\T(i)= t_0 $ then $a_i=1$ implies $q_i(D)>-\epsilon_2$, and $a_i=-1$ implies $q_i(D)<\epsilon_2$, where $q_i$ is the $i$-th query provided by $A$ when provided with answer sequence $a_1,\ldots, a_{i-1}$.

If a GD wrapper is $(\epsilon_1,\epsilon_2,\delta)$-accurate against any oracle, we simply say it is $(\epsilon_1,\epsilon_2,\delta)$-accurate.

%
It can be seen that a GD wrapper together with an \fla Oracle imply an oracle that answers  statistical queries (we provide the proof in \cref{sec:boolred})

\begin{lemma}\label{lem:boolred}
Suppose that there exists ($2\epsilon_1,\epsilon_2,\delta)$-accurate GD wrapper with learning $\eta>0$ that answers $T_1$ queries and perform $T_2$ iterations. 
Suppose also, that there exists an oracle  that is a  $(\epsilon_1,\delta)$-accurate \fla oracle, $\O_F$, that receives $m$ samples and answers $T_2$ adaptive queries against Gradient Descent with learning rate $\eta>0$.  

Then there exists an $(\epsilon_2,2\delta)$-accurate oracle, $\O$, that receives $m$ samples and answers $T_1$ adaptive queries against any Boolean analysts.
Moreover, if $\O_F$ is  $(\epsilon_1,\delta)$-post-hoc generalizing then $\O$ is $(2\epsilon_2,2\delta)$-post-hoc generalizing.
\end{lemma}

Next, in \cref{sec:construction} we provide a construction of a GD wrapper. Specifically, we prove the following:
\begin{lemma}\label{lem:main}
For sufficiently small $\epsilon,\eta >0$, and $\delta>0$. Assume $\eta<\sqrt{\epsilon/48}$, and $T<\min\{1/16 \eta,1/24\epsilon\}$. For sufficiently large $d$, there exists a $(\epsilon,O(\epsilon),\delta)$-accurate GD wrapper with a learning rate $\eta>0$ and $1$-Lipschitz wrapper function that answers $T$ queries and performs $2T$ iterations.
\end{lemma}

\subsection{Proof of \cref{thm:main2}}\label{sec:prfmain2}
Suppose $\eta=O(\epsilon)$, and set
$T'= \min\{1/16 \eta,1/24\epsilon\}=\Omega (1/\epsilon)$. We then use \cref{lem:main} and conclude from \cref{lem:boolred,lem:boolean} that if there exists an oracle that is $(\epsilon,\delta)$-accurate \fla and $(\epsilon,\delta)$-post-hoc generalizing oracle that receives $m$ samples and answer $T'$ questions, then there exists an $(O(\epsilon),O(\delta))$-accurate post-hoc generalizing oracle that receives $m$ samples and answers $T'/2$ questions against any Boolean analyst and answers $\Omega (T'/\log 1/\epsilon)=\Omega(1/\epsilon\log 1/\epsilon)$ against any analyst.

\subsection{Proof of \cref{lem:main}}\label{sec:construction}
\subsubsection{The Construction} We start the proof by first providing the construction. We let $\T(t)=2t$, and proceed to construct $f, \vec{\rho}$, and $\vec{\kappa}$.

\paragraph{Defining $f$:} We begin by defining for every $T$, a convex $4$-Lipschitz function recursively, where each $f_T$, is defined over $\reals^{4\cdot 2^{T-1} -3}$. Our final choice for the proof will be $f:=f_{T}$ (after some permutation). 

We start with $f_1$, which is defined over $\reals^5$ to be:
\[ f_1(A;w,x) = 
\max \left(h_1(q_1, w,x),h_2(w), \frac{3}{4}\eta\right),\]
where $q_1$ is the first query outputted by $A$, and:
\begin{align*} h_1(q; w,x) &= w(1)- \frac{(1+q(x))}{16}w(2)- \frac{(1-q(x))}{16}w(3), \\
h_2(w) &= \max\{(2+\xi_2)w(2)+w(3)- w(4), (2+\xi_3)w(3)+w(2)- w(5)\}.
\end{align*}
Where $\xi_2,\xi_3\sim U[0,\eta^4]$ are chosen uniformly at random.  

To define the rest of the sequence, we will notate the following three projection functions:
\begin{align*}
\Pi^{(t)}_{+}(w) &=(w(4),w(6),w(8),\ldots,w(4\cdot 2^{t-1}-4)\\
\Pi^{(t)}_{-}(w) &=(w(5),w(7),w(9)\ldots, w(4\cdot 2^{t-1}-3))\\
\Pi^{(t)}_{0}(w) &=(w(1),w(2),w(3),w(4),w(5)).
\end{align*}
To avoid cumbersome notations we will omit the dependence of the matrices $\Pi^{(t)}_{\pm}$ on $t$ when this can be derived from the dimension of the input.
Then, $f_t$ is defined as follows: draw, independently, a function $f_1$ and a function $f_{t-1}$ and define:
\[ \bar f_t(A;w,x) = \max \left(f_1(A;\Pi_0 w,x) ,f_{t-1}(A^+;\Pi_+ w,x), f_{t-1}(A^-;\Pi_{-}w,x)\right).\]

where $A^+$ denotes the analyst obtained from $A$ if we provide it with answer $a_1=1$ to $q_1$ and $A^-$ denote the analyst obtained by providing $a_1=-1$.

One can show by induction that each $f_t$ is indeed convex. This follows easily from the fact that maximum of convex functions is convex and that $h_1$ is convex, in fact linear, and also $h_2$ (which is again maximum over two linear functions). Next, recall that for a function $g=\max \{g_1,g_2,\ldots, g_m\}$ the subgradient at $w$ is given by $\mathrm{conv} \{ \nabla g_i: g_i(w)=g(w)\}$, hence by induction we prove that for every analyst $A$:
\[ \|\nabla \bar f_t(A;w,x)\| \le \max \{ \|\nabla h_1\|,\|\nabla h_2\|,\|\nabla f_{t-1}\|,\|\nabla f_{t-1}\|\} \le \max_{w} \|\nabla h_2(w)\|\le  4.\]

\paragraph{Defining $\kappa$:}
For the answering mechanisms we define for every $t$ a seqeunce $\bar \kappa^{(t)}=\{\kappa^{(t)}_i\}_{i=1}^t$ of $t$ functions from $\reals^{4t-3}$ to $[0,1]$, defined as follows. For every $t$:
\[ \kappa^{(t)}_1(g_1,g_2)=
\begin{cases}
\phantom{+}1, & \Pi_{0}g_2(4)<- \frac{1}{2} \\
-1, & \mathrm{else}
\end{cases},\]
and
\[\kappa^{(t)}_i(g_1,g_2,\ldots,g_{2t})=
\begin{cases}
\kappa^{(t-1)}_{i-1}(\Pi_{+}g_3,\ldots \Pi_+ g_{2t-2}) & \kappa^{(t)}_1(g_1,g_2) =1 \\
\kappa^{(t-1)}_{i-1}(\Pi_{-}g_3,\ldots \Pi_- g_{2t-2}) & \mathrm{else} \\
\end{cases},\]

\ignore{
\paragraph{Defining $\rho$}
Next, we define for every $t$ a seqeunce $\vec{\rho}^{(t)}= \{\rho^{(t)}_i\}_{i=1}^{2t}$ of $2t$ gradient access functions from $\reals^{4t-3}$ to $[0,1]$,
defined as follows:

First define $\tilde \rho^{(t)}_1,\tilde \rho^{(t)}_2$:
\[\tilde\rho^{(t)}_1(q,w,x)=(h_1(q,\Pi_0 w,x) ,\Pi^{\top}_0\nabla h_1(q,\Pi_0 w,x)),\quad  \tilde\rho^{(t)}_2(q,\Pi_0 w,x) =(h_2(\Pi_0 w), \Pi^{\top}_0\nabla h_2(\Pi_0 w)).\]
$\tilde \rho^{(t)}_{3},\ldots, \tilde \rho^{(t)}_{2T}$ are defined inductively as follows:

\[ \rho^{(t)}_j(q,c_1,\ldots, c_{\floor{j-2}},w,x) = \begin{cases}
\Pi^{\top}_+\rho^{(t-1)}_{j-2}(q,c_2,\ldots, c_{\floor{j-2}},\Pi_{+}w,x) & c_1=1
\\
\Pi^{\top}_-\rho^{(t-1)}_{j-2}(q,c_2,\ldots, c_{\floor{j-2}},\Pi_{-}w,x) & c_1=-1
\end{cases}
.\]
Finally, we let
\begin{equation}\label{eq:defrho}\bar \rho_i^{(t)} = \tilde \rho_i^{(t)}(q,a_1,\ldots, a_{\floor{i-2}},w_i,x).\end{equation}
Note that the answering sequence as well as $w_i$ are determined by $\O(1),\ldots, \O(j)$, (through $\kappa_t$),  hence $\bar \rho$ is well defined.
}
\paragraph{Defining an Oracle}
The final object we will need is a random oracle that will help ``hide" the analyst in the function $f$. It will rely on an embedding that we will later use and compose it with $f$, however for the proof it will be easier to think of the transformation as operating over $\O$. Specifically, let  $d\ge 4\cdot 2^{T}-3$ and define:
\[\Sigma_T:\reals^{4\cdot 2^T-3}\to \reals^{d},\] be a random embedding that simply permutes the coordinates (i.e $\Sigma_T$ embeds $\reals^{4\cdot 2^T-3}$ in $\reals^{d}$ by padding with zeros and then applying a random permutation), where all coordinates are randomly permuted except
\begin{equation}\label{eq:randSigma} \Sigma_T e^{4\cdot 2^T-3}_1 =e^d_1,\end{equation}
where we denote by $e^k_i$ the $i$-th standard basis vector in $\reals^k$. We will call such an embedding a \emph{random hiding with known first coordinate}
Next, given an embedding $\Sigma: \reals^{d_2}\to \reals^{d_1}$, and an \fla Oracle $\O_F$ in $\reals^{d_1}$, we define an \fla oracle $\O_{\Sigma}$ in $d_2$ that operates as follows. Given gradient access function $\rho_t$, $\O_{\Sigma}$ provides $\O_F$ with the gradient access function:
\[ \bar \rho_{t,\Sigma}(x) = (\bar\rho^0_t(x), \Sigma \bar \rho^1_{t}(x)),\] as in \cref{eq:rho0rho1}.
In turn, we define inductively the trajectory induced by $\O_F$, where $u_1=\eta e^d_1$, and $u_t= u_{t-1} -\eta \O_F(u_{t-1})$. The oracle $\O_{\Sigma}$ returns at step $t$,
\[ \O_{\Sigma}(w_t) = \Sigma^\top \O_F(u_t ).\]
Note that this response is dependent only on $\rho_t$ which makes it a valid \fla Oracle. 



\subsubsection{Continuing with the proof:}
The final proof of the Lemma builds upon the following result which we next set out to prove:
\begin{lemma}\label{lem:mainblah}
Fix $\epsilon<1/48$,assume $\eta <\sqrt{\epsilon/48}$, and $T<\min\{1/16 \eta,1/24\epsilon\}$. Also, let $d \ge  4\cdot 2^T+\frac{4\cdot 2^T+320}{\eta^6\delta}$. Suppose $\bar v\in \reals^d$ is a vector where
$\|\bar v\|\le 2\eta\epsilon$. 
and that \[\Sigma:\reals^{4\cdot 2^T-3}\to \reals^{d},\] is a random hiding with known first coordinate (as in \cref{eq:randSigma}).
Then, for any \fla Oracle $\O_F$, with probability $(1-2T\delta)$ (w.r.t the random embedding $\Sigma$):
The GD wrapper $(f_T(A,w,x),\bar \kappa^{(T)})$ is $(\epsilon,99\epsilon,1)$-accurate for the initialization point $w_1=\eta e_1+\Sigma_T^\top \bar v$ and answers $T$ queries with $2T$ iterations against the Oracle $\O_{\Sigma_T}$.
\end{lemma}
Before we proceed with the proof let us observe how it entails \cref{lem:main}.
For this, we show that the GD wrapper 
$(f\circ \Sigma_T, \vec{\kappa}_{\Sigma_T})$ with initialization $w_1= e^{d}_1$, is $(\epsilon,99\epsilon,2T\delta)$- accurate against $\O_F$, where:
\[f\circ \Sigma = f(A,\Sigma^\top w,x)\quad \textrm{ and:} \quad \kappa_{t,\Sigma} (g_1,\ldots,g_t)= \kappa_t( \Sigma^\top g_1 ,\ldots,  \Sigma^\top g_t).\]
Therefore by setting $\delta=\bar \delta/T$, $s=\eta e_1$ in $\reals^d$, we obtain a $(\epsilon, O(\epsilon), \bar \delta)$-accurate wrapper against an arbitrary \fla oracle $\O_F$. Note that we can rescale $f$ to obtain a $1$-Lipschitz function and still remain with $(\epsilon, O(\epsilon), \bar \delta)$-accurate wrapper (up to new constant factors).

So we are left with showing that $(f\circ \Sigma_T, \vec{\kappa}_{\Sigma_T})$ is indeed accurate against $\O_F$. For this, let $\bar u_1,\ldots \bar u_T$ be the trajectory of the wrapper against the oracle, and we assume that for every $t\le t_0$ the gradients are accurate:
 \[ \|\O_F(\bar u_t)- \Sigma \nabla f(A,\Sigma^\top \bar u_t, x)\|\le \epsilon,\]
 and we need to show accuracy of the answers. Now, let $w_1,\ldots, w_{t_0}$ be the trajectory of the wrapper $(f,\kappa)$ against $\O_{\Sigma}$. We will show that \begin{equation}\label{eq:show1}\O_{\Sigma}(w_t) = \Sigma^\top \O_F(\bar u_t),\end{equation} and accuracy of the gradients here, namely:
  \begin{equation}\label{eq:show2} \|\O_\Sigma(w_t)- \nabla f(A,w_t, x)\|\le \epsilon.\end{equation}
Then, by \cref{lem:mainblah} we have accuracy of the answers $\kappa_t(\O_\Sigma (w_1),\ldots, \O_\Sigma (w_{t_0}))=\kappa_{t,\Sigma} (\O_F(\bar u_1),\ldots, \O_F(\bar u_{t_0}))$  and we are done.
 
  We show that \cref{eq:show1,eq:show2} hold by induction. In particular, we assume that \cref{eq:show1,eq:show2} hold and that also $\bar u_{t'}=u_{t'}$ and $w_{t'}=\Sigma^\top u_{t'}$, for any $t'< t$ (the induction base is trivial).

Importantly, because $\|\bar u_1\|\le \eta$, $T< 1/16\eta$ and $f$ is $4$-Lipschitz, then one can show that for every $t$, $\|\bar u_t\|\le 1$, and 
\[ \bar u_t = \bar u_{t-1} - \eta \O_F(\bar u_t).\]
We then have:
\[\bar \rho_{t',\Sigma} (x) = (f(A,w_{t'},x),\Sigma \nabla f(A,w_{t'},x))= (f(A,\Sigma^\top u_{t'}, x),\Sigma \nabla f(A,\Sigma^\top u_{t'},x)) = (f(A,\Sigma^\top \bar u_{t'}, x),\Sigma \nabla f(A,\Sigma^\top \bar u_{t'},x)),\]
By standard chain rule, then, $\rho_{t',\Sigma}$ are the gradient access functions for $\O_F$ when it interacts with $f\circ \Sigma$ hence $\O_F(\bar u_{t'})= \O_F(u_{t'})$\footnote{where in the LHS $F$ interacts with the wrapper $f\circ \Sigma$ and in the RHS its state is goverened by the inputs provided by $\O_{\Sigma}$ as depicted}. We then have in turn,
\[ w_{t} = \Pi(w_{t-1}-\eta \O_{\Sigma}(w_{t-1})) = \Pi(\Sigma^\top u_{t-1} - \eta \Sigma^\top \O_F(u_{t-1}))= \Sigma^\top \left(\bar u_{t-1}-\eta \O_F(\bar u_{t-1})\right) = \Sigma^\top \bar u_{t}, \]
and that $\bar u_{t} = \bar u_{t-1} -\eta \O_F(\bar u_{t-1}) = u_{t-1} - \eta \O_F(u_{t-1}) = u_{t}$. By the same reasoning as above we have then that $\bar \rho_{t,\Sigma}$ is the gradient access function provided to $\O_F$ and 
\[ \O_{\Sigma}(w_t) = \Sigma^\top \O_F(u_t)=\Sigma^\top \O_F(\bar u_t).\]


\begin{proof}[of \cref{lem:mainblah}]
We prove the statement by induction. 
 Again, we recall that by our choice of parameters $T\le 1/16\eta$, $\|w_1\|\le \eta +2\eta \epsilon$, and we have for every $w_t$, $\|w_t\|\le 2\eta  + \eta T\le 1/16+1/4\le 1$. Therefore, applying GD over the sequence never sets outside of the unit ball, and the update steps produced by GD, under the assumption that the gradients are approximates hence bounded, is given by:
\begin{equation}\label{eq:upwithout} w_{t+1}=w_t-\eta \O_{\Sigma}(w_t).\end{equation}

\paragraph{The base case $T=1$.} 
For the base case, we assume that
\begin{equation}\label{eq:approx}\| \O_{\Sigma}(w_t)-\E_{x\sim D} [\nabla f_1(A;w_t,x)]\| \le \epsilon,\end{equation}
and we need to verify the accuracy of answers for $t_0=2$.

First, notice that at initialization $w_1=\eta e_1 + \Sigma^\top  v$, and $\|v\|\le 2\eta \epsilon$. For $j\ge i$, let us denote by $\Pi_{i:j}$ a projection over coordinates $i$ to $j$, then note that for $i\ge 2$, since $\Pi_{i,j}\Sigma^\top$ is chosen randomly then
\begin{equation}\label{eq:showedbefore1} \E[\|\Pi_{2,5}\Sigma^\top  v\|^2]\le = \sum_{i=2}^5 \E[(e_{i}^\top \Sigma^\top v)^2] =
\sum_{i=2}^5\frac{1}{d}\sum_{j=2}^{d} v(j)^2 
\le \frac{10\epsilon}{d}\le \frac{\epsilon\eta^6}{32} \delta\le \eta^4\delta.\end{equation}
By Markov's inequality, we have that with probability at least $(1-\delta)$ we have that
\[\|\Pi_{2,5}\Sigma^\top v\|\le \eta^2 .\]

If this happens, as $h_2$ is $4$-Lipschitz, $h_2(0)=0$ and $h_2$ depends only on the $2$ to $5$ coordinates, then \[|h_2(w)\|\le 4 \eta^2 \le \frac{1}{2}\eta.\] 

On the other hand, because $h_1$ is $2$-Lipschitz, one can also observe that \begin{align*}h_1(q_1,w_1,x)&\ge h_1(q_1,\eta e_1,x) -2\|w_1-\eta e_1\|\\
&=
h_1(q_1,\eta e_1,x) -2\|\Sigma^\top v\|\\
&\ge  \eta- 2\eta\epsilon \\
&> \frac{3}{4}\eta.\end{align*} Hence,
\begin{equation}\label{eq:w1} f_1(A;w_1,x)= h_1(q,w_1,x),\quad \mathrm{and}\quad \nabla f_1(A;w_1,x)= \nabla h_1(q,w_1,x).\end{equation}
Next, for $t_0=2$ note that by assumption
\begin{align*}\eta\epsilon &\ge \eta  \cdot \| \O_\Sigma(w_1)- \E_{x\sim D}[\nabla f(A;w_1,x)]\|\\
&=\left\|w_2- \left(w_1-\eta \E_{x\sim D}[\nabla f(A;w_1,x)]\right)\right\|\\
&=\left\|w_2 -(\eta e_1 + \Sigma^\top v) +\eta \E_{x\sim D}(\nabla h_1(q,w_1,x)\right\| & \textrm{\cref{eq:w1}}\\
&= \left\|w_2 -\eta e_1 - \Sigma^\top v +\eta\E_{x\sim D}[e_1 -\frac{1+q(x)}{16} e_2 -\frac{1-q(x)}{16}e_3]\right\|\\
&=\left\|w_2 -  \Sigma^\top v -\frac{1+q(D)}{16}\eta e_2 -\frac{1-q(D)}{16}\eta e_3\right\|\\
&\ge 
\left\|w_2 -\frac{1+q(D)}{16}\eta e_2 -\frac{1-q(D)}{16}\eta e_3]\right\|-2\eta\epsilon \\
&=\left\|w_2 -\bar{w}_2\right\|-2\eta\epsilon
.\end{align*}
where we denote $\bar{w}_2=\frac{1+q(D)}{16}\eta e_2 +\frac{1-q(D)}{16}\eta e_3$. Hence,
\begin{equation}\label{eq:w2bound} \|w_2-\bar w_2\| \le 3\eta \epsilon\end{equation}

Now, by $4$-Lipschitness of $h_2$:
\begin{align*}
    h_2(w_2)&\ge h_2(\bar w_2)-4\|w_2-\bar w_2\|\\
    &\ge \max\{(2+\xi_2)(1+q(D)) +(1-q(D)),(2+\xi_3)(1-q(D)) +(1+q(D))\} \eta -12\eta\epsilon \\
    &\ge \max\{3+ (1+\xi_2)q(D),3-(1+\xi_3)q(D)\}\eta -12\eta \epsilon 
    \\
    &\ge 3\eta -12\eta\epsilon \\
    &\ge \frac{3}{4}\eta. & \epsilon \le 1/48
\end{align*}
On the other hand, \[h_1(q,w_2,x)\le h_1(q,\bar w_2,x)+12\eta\epsilon \le \left|\frac{1+q(x)}{16}\frac{1+q(D)}{16}\right|\eta+\left|\frac{1-q(x)}{16}\frac{1-q(D)}{16}\right|\eta+ 12\eta\epsilon\le \frac{\eta}{32} +\frac{\eta}{4}<\frac{\eta}{2}.\]
We obtain then that for $t_0=2$:
\begin{equation}\label{eq:w2} f_1(A;w_2,x)= h_2(q,w_2,x),\quad \mathrm{and}\quad \nabla f_1(A;w_2,x)= \nabla h_2(q,w_2,x).\end{equation}
 We now need to show that the answer is accurate, contingent on \cref{eq:approx} holding for $t_0=2$.
Now we've already shown that 
\[ \E_{x\sim D}[\nabla f_1(A;w_2,x)]= \E_{x\sim D}[\nabla h_2(w)]=\nabla h_2(w_2),\]

Now, $h_2$ is random, and note that $w_2$ is independent of the random bits of $h_2$ as it is determined by the oracle $\O$, $w_1$ and $\bar \rho_1$, all are independent of $\xi_{2,3}$. The randomness ensures that with probability $1$, the two terms are not equal, and the subgradient at $w_2$ is well defined. Next, notice that for $g_2=\O_\Sigma(w_2) = \nabla h_2(w_2) +O(\epsilon)$, by assumption. We have, then, $g_2(4)<-1/2$ if and only if the first term is greater. Suppose, then, that the first term is greater.

Denote $\|\Pi_{4:5}(w_2-\bar w_2)\|=\xi$.
Then by the definition of $h_2$:
\[ (2+\xi_2)w_2(2)+ w_2(3)-w(4) > (2+\xi_3)w_2(3)+w_2(2)-w(5) .\]
By rearranging terms and noting that $\xi_2,\xi_3\sim U[0,\eta^3]$ we have that:
\[ w_2(2) > w_2(3) - 2\eta^3+w(4)-w(5).\]
which implies since $|w_2(2)-\bar w_2(2)|\le 3\eta\epsilon$, similarly $w_2(3)$, and $|w(4)|,|w(5)|\le \xi$:
\[ \frac{\eta (1+q_1(D))}{16} >  \frac{\eta(1-q_1(D))}{16} - 2\eta^3 -6\eta \epsilon -2\xi.\]
Rearranging terms:
\[ q_1(D) > -16\eta^2 -48\epsilon-\frac{2\xi}{\eta}>-49\epsilon-\frac{2\xi}{\eta},\]
where last inequality follows from $\eta<\sqrt{\epsilon/48}$. A similar calculation show that the reversed inequality holds if the inequality is reversed. We now only need to show that with probability at least $1-\delta$ $\xi<2\eta\epsilon$, this is similar to the way we bounded $\|\Pi_{2:5}v\|$. In particular note that
\[\xi = \|\Pi_{4:5}w_2\| = \|\Pi_{4:5} \Sigma^\top v+\eta \Pi_{4:5}O_{\Sigma}(w_2)\|\le \eta^3 + \eta \|\Pi_{4:5}\Sigma^\top \O_F(u_2)\|\le \eta\epsilon + \eta \|\Pi_{4:5}\Sigma^\top \O_F(u_2)\|,\]
where the last inequality is true since $\eta^2<\epsilon$. Next, if we set $v_2=\O_F(u_2)$, then $v_2$ is determined by $\bar \rho_1$, $w_1$ and $w_2$, conditioning on these vectors we know that $\Pi_{4:5}\Sigma^\top$ is a random embedding and permutation in $\reals^{d-3}$, and, by choice of $d$, we have that
\[ \|\Pi_{4:5}\Sigma^\top v_2\|^2 =\sum_{i=4}^5 (e_i^\top \Sigma^\top v_2)^2\le \frac{2}{d-3} \sum v_2(i)^2 \le \frac{2\epsilon}{d-3}\le \frac{2\epsilon\eta^3\delta^2}{4\cdot 2^T}\le  \epsilon^2\delta,\]
where, again, the last inequality follows from $\eta^2<\epsilon$. By Markov inequality we obtain the desired bound on $\xi$, which concludes the base case.

\paragraph{The induction step:} We now move on to the induction step. Suppose we proved the statement for $T-1$, and we want to prove it for $T$. We begin with the observation that, similar to the argument that bounds $\|\Pi_{2:5}v\|$, we have that since $d> \frac{4\cdot 2^T}{\eta^4\delta}$:
\[ \E[\|\Pi_{\pm}\Sigma^\top v\|^2]\le \eta^6\delta.\]
One can also observe, for the same reason,
\[ \E[\|\Pi_{\pm}\Sigma^\top \O_F(u_1)\|^2] \le \eta^6\delta,\]
Hence with probability $(1-2\delta$):
\begin{equation}
    \max\left\{\|\Pi_{\pm}\Sigma^\top \O_F(u_1)\|,\|\Pi_{\pm}\Sigma^\top v\|\right\}\le \eta^3\le \frac{\eta\epsilon}{32}.
\end{equation}
We assume that this event happened. It is then easy to see that for $i=\{1,2\}$ we have that
\begin{equation}\label{eq:fT} f_T(A,w_i,x)= \max\{f_1(A,\Pi_0w_i,x),f_{T-1}(A^+,\Pi_+w_i,x),f_{T-1}(A^-,\Pi_- w_i,x)\}= f_1(A,\Pi_0w_i,x).\end{equation}
As such the calculations of the first two iterates are the same as in the base case and we have that for $t_0=\{1,2\}$, the conditions are met. 

We now claim that for $t\ge 3$ we have that, if
\begin{equation}\label{eq:ismet} \|\O_{\Sigma}(w_{t'})-\E_{x\sim D}[\nabla f_T(A,w_{t'},x)]\|\le \epsilon, \quad {t'=1,\ldots, t},\end{equation}
then:
\begin{equation}\label{eq:fT3}f_T(A,w_t,x)=  
\begin{cases}
f_{T-1}(A,\Pi_+ w_t,x) & a_1=1\\
f_{T-1}(A,\Pi_- w_t,x) & a_1=-1 
\end{cases}\end{equation}
Let us consider the case $a_1=1$ (the other case is equivalent).  We've already shown that if $a_1=1$ then
\[ \nabla f_{T}(A;w_2,x) = \nabla f_{1}(A;\Pi_0 w_2,x) =\nabla h_2(w_2).\]

Hence, let 
\begin{align*}\bar w_3 &= \bar w_2 - \eta \nabla h_2(w_2)\\
&=\left(\frac{(1+q_1(D))}{16}-2-\xi_2\right)\eta e_2 + \left(\frac{(1-q_1(D))}{16}-1\right)\eta e_3 + \eta e_4 \\
\end{align*}
since
\[ \|\bar w_3-w_3\|\le \|\bar w_2- w_2\|+ \eta \|\O_{\Sigma}(w_2)-\E[\nabla f(A,w_2,x)]\| \le 4\eta\epsilon,\] 
we conclude that we can write 
\[w_3 = \eta e_4 + \eta v_3,\]
where $-3\eta-4\eta\epsilon\le  v_3(2)\le 0$, $-\eta-4\eta\epsilon \le v_3(3)\le 0$, and $v_3(1)\le 4\eta \epsilon$.

By definition of $h_1,h_2$ we have then:
\begin{equation}\label{eq:h1dni} h_1(q_1,w_3,x) \le v_3(1) + \frac{1}{8}|v_3(2)|+ \frac{1}{8}|v_3(3)|\le 4\eta \epsilon + \frac{3+4\epsilon}{8}\eta +\frac{1+4\epsilon}{8}\eta\le \frac{1}{2}\eta + 5\eta\epsilon<\frac{3}{4}\eta ,\end{equation}
and,
\begin{equation}\label{eq:h2dni} h_2(w_3) \le 0.\end{equation}
Taken together $f_1(A,w_3,x)< \frac{3}{4}\eta$. We also have that $f_{T-1}(A,\Pi_{-} w_3,x)<\frac{3}{4}\eta$. Indeed, since $\Pi_{-} \bar w_3 = 0$, we have that $\|\Pi_{-} w_3\| \le 4\eta \epsilon$, and one can observe that then
\[f_{T-1}(A^-,\Pi_{-} w_3,x)<\frac{3}{4}\eta.\]
It can also be seen, from the construction of $f_{T-1}$ and the fact that $w_3 = \bar w_3 + n$ where $\|n\|\le 4\eta \epsilon$ that
\[ f_{T-1}(A^+,\Pi_{+} w_3,x) > \frac{3}{4} \eta,\]
and
\[f_T(A,w_3,x)=  f_{T-1}(A^{+},\Pi_+ w_3,x), \textrm{and} \nabla f_T(A,w_3,x) = \nabla f_{T-1}(A^{+},\Pi_+ w_3,x) \]

Now we want to claim that this qualities remains for the whole length of the sequence. For this we need to show that for every $t\ge 3$:
\[ \max\{f_1(\Pi_0, w_t,x) ,f_{T-1}(A^{-},\Pi_- w_t,x)\} \le \frac{3}{4}\eta,\]
Indeed, if this holds, since $f_{T-1}(A^{+},\Pi_+ w_t,x) \ge \frac{3}{4}\eta$ and since whenever $f_{T-1}(A^{+},\Pi_+ w_t,x) =f_1(\Pi_0, w_t,x)= f_{T-1}(A^{-},\Pi_- w_t,x)= \frac{3}{4}\eta$ then $\nabla f_T= \nabla f_{T-1}(A^{+},\Pi_+ w_t,x) =0$, then this suffice.

First, it is easy to see that $f_{T-1}(A^{-},\Pi_- w_t,x)\le \frac{3}{4}\eta$. Indeed, for every $t$, $w_t$ is determined by the gradient access functions $\{(f_{T-1}(A,w_t,x), \Sigma \nabla f_T(A,w_t,x))\}_{t'\le t}$, which by induction depends on elements of the form $(f_{T-1}(A^{+}, \Pi_{+} w_{t'}), \Sigma \Pi_+^\top \nabla f_{T-1}(A^+, \Pi_{+} w_{t'},x))$ and $(h_1(\Pi_0,w_1),\Sigma \Pi_0^\top \nabla h_1(\Pi_0 w_1))$ and $h_2(\Pi_0,w_2), \Sigma \Pi_0^\top \nabla h_2(\Pi_0 w_2))$.

It can be seen, that given these vectors, the embedding $\Sigma \Pi_{-}^\top$ can thought of as a random embedding in a space of dimension at least $d-5- 4\cdot 2^{T-1}+7\ge \frac{4\cdot 2^{T-1} +7}{\eta^6\delta}$, (in other words, there is a subspace in $d$ of the given size where the coordinates are embedded and permuted randomly).  In particular,
\[\E[ \|\Pi_{+} \Sigma ^\top \O_F(u_t)\|^2 ] \le \sum_{i=1}^{4\cdot 2^{T-2}-7} \E[(e_i^\top \Sigma^\top \O_F(u_t))^2] = 4 \frac{4\cdot 2^{T-2}-7}{d-5- 4\cdot 2^{T-1}+7}\le \eta^6\delta.\]

and
\[ \nabla f_T(A,w_1,x) = \nabla f_{1}(A,\Pi_0 w_1 , x) = e_1 -\frac{1+q(x)}{16}e_2 - \frac{1-q(x)}{16}e_3.\] 

In particular, is determined by $\O_F(w_3)$ $\bar \rho_{2,\Sigma}(x) = \Sigma \Pi_0 \nabla f_T(A,w_2,x)$, and $\bar \rho_{2,\Sigma}(x)= \Sigma\Pi_0  \nabla f_T(A,w_1,x)$.

We begin with the observation that, given $\Sigma\Pi_0$, the embedding $\Sigma \Pi_{+}$ is still a random embedding in $\reals^{d-5}$, with the exception that the embedding of $\Sigma e_4$ is determined by $\Sigma \Pi_0$ but since all other coordinates are orthogonal to the range $\Pi_0$, we have that all other coordinates in the subspace are randomly embedded and permuted in $\reals^{d-5}$. A similar argument holds for $\Sigma \Pi_{-}$ with $\Sigma e_5$. As $w_3$ is determined by $\Sigma \Pi_0$

and thus, we can treat $\Sigma \Pi_{+}$ and $\Sigma \Pi_{-}$ as random hidings where the first coordinate is known to dimension $d>\frac{4\cdot 2^T}{\eta^6\delta}-5> \frac{4\cdot 2^{T-1}}{\eta^6\delta}$.

Now we prove by induction that for all $t>3$ we have that the equality remains. This can be seen by the fact that the gradient is almost completely orthogonal to the coordinates on which the other terms depend on except $w(4)$, but $f_{T-1}(A,\Pi_+ w_3,x)$ is decreasing in $w(4)$, hence we conclude that $\Pi_0 \nabla f_{T-1}(A,\Pi_+ w_t,x)\ge 0$, which means that $\Pi_0 w_t$ can increase at each iteration by at most $\eta \epsilon$. So we have that:

\[ \|\Pi_0 w_t -\Pi_0 w_3\|\le \eta \epsilon T,\]
and by 4-Lipschitzness of these terms as well as \cref{eq:h1dni,eq:h2dni}
\[ \max\{h_1(q_1,\Pi_0 w_t,x),h_2(\Pi_0 w_t)\} \le 4\eta\epsilon T+ \max\{h_1(q_1,\Pi_0 w_3,x),h_2(\Pi_0 w_3)\}\le \frac{1}{2}\eta +5\eta\epsilon(T+1) \le  \frac{3}{4}\eta,\]
where the last inequality follows from $\epsilon \le 1/24$, $T\le 1/(24\epsilon)$.

This proves that \cref{eq:fT3} holds. Now we can use the induction hypothesis. First, notice that if \cref{eq:ismet} is met, then by our assumption and contraction of a projection operator we have for $t'=1,\ldots, t$: 
\begin{align}\label{eq:ismetmet}
\|\O_{\Sigma\Pi_+^\top}(w_{t'})-\E_{x\sim D}[\nabla f_{T-1}(A^{+},\Pi_{+}w_{t'},x)]\|
&=
\|\Pi_+\O_{\Sigma}(w_{t'})-\E_{x\sim D}[\Pi_{+}\Pi_{+}^\top\nabla  f_{T-1}(A^{+},\Pi_{+}w_{t'},x)]\|\\
&\le
\|\O_{\Sigma}(w_{t'})-\E_{x\sim D}[\Pi_{+}^\top\nabla  f_{T-1}(A^{+},\Pi_{+}w_{t'},x)]\|\\
&=
\|\O_{\Sigma}(w_{t'})-\E_{x\sim D}[\nabla  f_{T}(A^{+},\Pi_{+}w_{t'},x)]\|\\
&\le \epsilon, \quad ,\end{align}

Also, we have that the sequence $\Pi_{+}w_3,\ldots,\Pi_{+}w_T$ initializes at \[\Pi_{+} w_3= \Pi_{+}\bar w_3 +\Pi_{+}w_3-\Pi_{+}\bar w_3=\eta e_4 + \Sigma^\top v,
\]
where
\begin{align*}\|v\|&=\|\Sigma \Pi_{+}\bar w_3-\Sigma \Pi_{+}w_3\| \\
&=\|\Pi_{+}\bar w_3-\Pi_{+}w_3\| \\
&\le \|\Pi_{+}\bar w_2-\Pi_{+}w_2\| + \eta \|\O_{\Sigma}(2)-\E_{x\sim D}[\nabla f(A,w_{2},x)]\|\\
&\le \|\Pi_{+}w_2\| +\eta\epsilon\\
& = \| \Pi_{+}\Sigma^\top v\|+ \|\Pi_{+}\O_{\Sigma}(1)\| + \eta \epsilon &w_2=w_1-\eta \O_{\Sigma}(1). \\
&\le 2\eta^3 +\eta\epsilon \\
&\le 2\eta \epsilon.
\end{align*}
Also, note that:
\begin{align*}\kappa^{(T)}_i(\O_{\Sigma}(1),\ldots, \O_{\Sigma}(\T(i))& = \kappa_{i-1}^{(T-1)}(\Pi_{+}\O_{\Sigma}(w'_1),\ldots, \Pi_{+}\O_{\Sigma}(w'_{\T(i)}))\\
&=
\kappa_{i-1}^{(T-1)}(\O'_{\Sigma\Pi_{+}^\top}(w'_1),\ldots, \O'_{\Sigma\Pi_{+}^\top}(w'_{\T(i)}))
,\end{align*}
So to summarize, the sequence $\Pi_{+}w_3,\ldots, \Pi_{+}w_T$ and $a_2,\ldots, a_T$, is the trajectory sequence and answering sequence of the GD wrapper of $f_{T-1},\kappa^{(T-1)},\rho^{(T-1)}$, against the oracle $\O_{\Sigma \Pi_+^\top}$

If \cref{eq:ismetmet} is met then, we obtain that the sequence $a_2,\ldots, a_T$ is indeed correct,
with probability $1-2\delta (T-1)$. Applying union bound we obtain the desried result.
\end{proof}

\subsection{Proof of \cref{lem:boolred}}\label{sec:boolred}

By assumption, there exists a $(2\epsilon_1,\epsilon_2,\delta)$-GD wrapper with $\eta>0$ that answers $T_1$ queries and perform $T_2$ iterations. Without loss of generality we may assume that the wrapper initializes at $w_1=0$. We also assume the existence of an $(\epsilon_1,\delta)$ \fla Oracle $\O_F$. 

Given the \fla oracle and an analyst $A$ we define a statistical query oracle that works as follows:
\begin{itemize}
    \item Receive query $q_1$ from analyst $A$.
    \item For $t=1,\ldots, T_1$
    \begin{itemize}

     \item For $s= \T(t-1)+1,\ldots, \T(t)$ (where we define $\T(0)=1$)
        \begin{itemize}
            \item Provide the gradient access function $ \bar \rho_s(x)=(f(A,w_s,x),\nabla f(A,w_s,x))$ to $\O_F$.        
            \item set $g_s=\O_F(w_s)$.
            \item Set $w_{s+1}= w_s-\eta g_s$.
        \end{itemize}
        \item Provide the the answer $a_t = \kappa_{t}(g_1,\ldots,g_t)$ and receive $q_{t+1}$.
    \end{itemize}
\end{itemize}

Next, for the proof let us fix the random bits of the wrapper, the analyst and oracles. And we assume that $w_1,\ldots, w_T$ is the sequence generated by the recursive relation

\[w_1=0,\quad w_{t+1}=\Pi(w_t-\eta g_t),\]
where $\O_F$ at each iteration is provided with a function $\hat \rho_t(x) = (f(A,w_t,x),\nabla f(A,w_t,x))$. Now Suppose that the internal random bits of $\O_F$ are such that for all $t\le T_2$:
\[ \|g_t - \E_{x\sim D} [\nabla f(A;w_t,x)]\| \le \epsilon_1.\]
\begin{equation}\label{eq:flapost} \|\E_{x\sim D} [\nabla f(A;w_t,x)]- \frac{1}{m}\sum_{i=1}^m [\nabla f(A;{w}_t,x_i)]\| \le \epsilon_1.\end{equation}
This event happens with probability $(1-2\delta)$, and we throughout assume that it occured.

Under this assumption, by the property of an accurate GD wrapper, we can show by induction that the sequence of answers is $\epsilon_2$--accurate w.r.t the population and $2\epsilon_2$ accurate w.r.t the empirical distribution which can be seen to entail $(\epsilon_2,\delta_1)$-accuracy and $(2\epsilon_2,\delta_1)$-post-hoc generalization as in \cref{def:accF,def:posthocF}.
\paragraph{Acknowledgements}
The author would like to thank Tomer Koren and Uri Stemmer for various discussions on the topic. The author is supported by the Israel Science Foundation (grant number 2188/20), and by a grant from Tel Aviv University Center for AI and Data Science (TAD) in collaboration with Google, as part of the initiative of AI and DS for social good.
\bibliographystyle{abbrvnat}
\bibliography{bio}
\appendix
\section{Proof sketch for \cref{thm:amir}}\label{prf:amir}
The construction is a simplification of the construction in \cite{amir2021sgd}. For every $x\in \{0,1\}^d$ we define

\[ f(w,x) = \frac{1}{2}\sqrt{\sum x(i)\max\{0,w(i)\}^2} + \gamma v_x\cdot w,\]
where
\[v_x(i) = \begin{cases} -\frac{1}{m}& x(i)=0\\
 1 & x(i)=1\end{cases}.\]
and $\gamma< 1/20$ is sufficiently small so that the function $f$ is $1$-Lipschitz.

We define the distribution on $x$ so that every $x(i)$ are i.i.d Bernoulli variables.
Initializing at $w_1=0$ we have that $\nabla f(w_1,x) = \gamma v_x$, and
\[w_2= w_1 -\frac{\eta}{m}\sum_{j=1}^m v_{x_j},\]
Notice that $w_2(i)>0$ if and only if $x_j(i)=0$ for all $j=1,\ldots, m$, hence 
\[ x_j(i) \nabla \max\{0, w_2(i)\} = 0,\] for all $i$ and $j$. In turn:
\[ \|\nabla f(w_2,x_i)\| =\|\gamma \frac{1}{m} \sum_{j=1}^m v_{x_j}\| \le \gamma  \le 1/20.\]
On the other hand, suppose $w_2$ is such that there are coordinates such that $x_j(i)=0$ for all $j=1\ldots, m$, then:
\[\| \nabla F(w_2) \| \ge \left\|\E_{x\sim D} \left[\sum_{\{i:w_2(i)>0\}}\frac{x(i)\cdot w_2(i)}{\sqrt{\sum_{\{i:w_2(i)>0\}} x(i)\cdot w(i)^2}}e_i\right]\right\|- \gamma m \ge 1-\gamma m \ge 19/20.\]
Note that if $d$ is sufficiently large then, indeed, with probability $1/2$ there will be some coordinate $i$, where $x_j(i)=0$ for all $j$.

\end{document}